\documentclass{article}

\usepackage{amsthm}
\usepackage[utf8]{inputenc} % allow utf-8 input
\usepackage[T1]{fontenc}    % use 8-bit T1 fonts
\usepackage{hyperref}       % hyperlinks
\usepackage{url}            % simple URL typesetting
\usepackage{booktabs}       % professional-quality tables
\usepackage{amsfonts}       % blackboard math symbols
\usepackage{nicefrac}       % compact symbols for 1/2, etc.
\usepackage{microtype}      % microtypography
\usepackage{xcolor}         % colors
\usepackage{makecell}
\usepackage{multirow}
\usepackage{float}
\usepackage{threeparttable}
\usepackage{enumitem}

\usepackage{wrapfig}
\usepackage{newfloat}
\usepackage{listings}
\usepackage{bbding}
\usepackage{fancyhdr}
\usepackage{xcolor}
\usepackage{tablefootnote}
\usepackage{stfloats}
\usepackage{bm}
\usepackage{graphicx}
\usepackage{comment}
\usepackage{microtype}
\usepackage{graphicx}
\usepackage{subfigure}
\usepackage{booktabs} % for professional tables

\usepackage{hyperref}

\usepackage[accepted]{icml2024}

\usepackage{amsmath}
\usepackage{amssymb}
\usepackage{mathtools}
\usepackage{amsthm}

\usepackage[capitalize,noabbrev]{cleveref}

\theoremstyle{plain}
\newtheorem{theorem}{Theorem}[section]

\newtheorem{lemma}[theorem]{Lemma}

\theoremstyle{definition}

\theoremstyle{remark}

\usepackage[textsize=tiny]{todonotes}

\icmltitlerunning{Submission and Formatting Instructions for ICML 2024}

\begin{document}

\twocolumn[
\icmltitle{Optimizing Sequential Recommendation Models with Scaling Laws and Approximate Entropy}



\icmlsetsymbol{equal}{*}

\begin{icmlauthorlist}
\icmlauthor{Tingjia Shen}{sch}
\icmlauthor{Hao Wang}{sch}
\icmlauthor{Chuhan Wu}{comp}
\icmlauthor{Jin Yao Chin}{comp}
\icmlauthor{Wei Guo}{comp}
\icmlauthor{Yong Liu}{comp}
\icmlauthor{Huifeng Guo}{comp}
\icmlauthor{Defu Lian}{sch}
\icmlauthor{Ruiming Tang}{comp}
\icmlauthor{Enhong Chen}{sch}

\icmlaffiliation{comp}{Huawei Noah’s Ark Lab}
\icmlaffiliation{sch}{University of Science and Technology of China}
\end{icmlauthorlist}

\icmlcorrespondingauthor{Hao Wang}{wanghao3@ustc.edu.cn}
\icmlcorrespondingauthor{Enhong Chen}{cheneh@ustc.edu.cn}

\icmlkeywords{Recommendation System, Scaling Law}

\vskip 0.3in
]
\printAffiliationsAndNotice{}



\begin{abstract}
%final check:1 anonymity 2. proceeding in icml2024.sty
Scaling Laws have emerged as a powerful framework for understanding how model performance evolves as they increase in size, providing valuable insights for optimizing computational resources. In the realm of Sequential Recommendation (SR), which is pivotal for predicting users' sequential preferences, these laws offer a lens through which to address the challenges posed by the scalability of SR models. 
However, the presence of structural and collaborative issues in recommender systems prevents the direct application of the Scaling Law (SL) in these systems.
In response, we introduce the Performance Law for SR models, which aims to theoretically investigate and model the relationship between model performance and data quality. Specifically, we first fit the HR and NDCG metrics to transformer-based SR models. Subsequently, we propose Approximate Entropy (ApEn) to assess data quality, presenting a more nuanced approach compared to traditional data quantity metrics.
Our method enables accurate predictions across various dataset scales and model sizes, demonstrating a strong correlation in large SR models and offering insights into achieving optimal performance for any given model configuration.
%final check:1 anonymity 2. proceeding in icml2024.sty

\end{abstract}

\section{Introduction}
\begin{comment}

%conclision
\end{comment}
Sequential Recommendation (SR), focused on suggesting the next item for a user based on their past sequential interactions to capture dynamic user preferences~\cite{wang2021decoupled,wang2019mcne,xie2024bridging,wang2021hypersorec}, has gained significant attention in commercial, internet, and diverse scenarios~\cite{DBLP:journals/tois/FangZSG20,DBLP:journals/corr/HidasiKBT15}. However, as the volume of user data increases, more expansive recommendation models are being implemented. The high computational requirements of these recommendation models~\cite{ding2023computational} lead to considerable expenses and unpredictability during development. This places added stress on developers to allocate resources efficiently~\cite{acun2021understanding} and GPU consumption~\cite{lin2022building}. To anticipate how recommendation models will perform without executing full-scale experiments, researchers have crafted a range of scaling laws to evaluate the models' effectiveness in various scenarios.

\begin{figure}[t]
  %%%%\vspace{-2mm}
  \centering
  \includegraphics[width=0.95\linewidth]{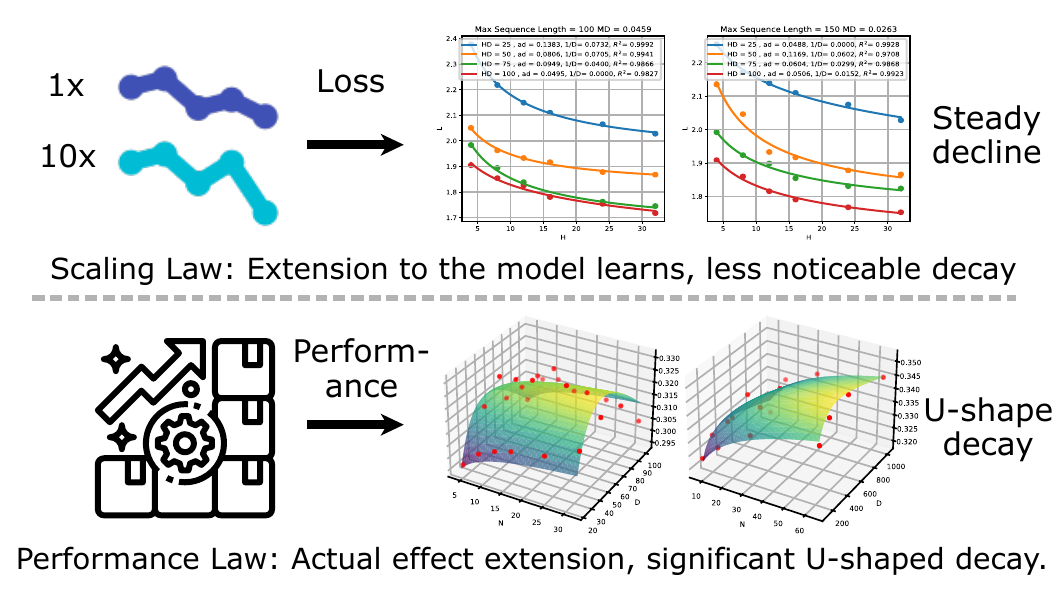}
  %%%%\vspace{-2mm}
  \caption{
  Distinction between Performance Law and Scaling Law. Performance typically shows decay as the model size increases.}
  \label{only_paint_results_NG}
  \vspace{-4mm}
\end{figure}

Scaling laws were first explored in the context of Large Language Models~\cite{kaplan2020scaling}. Specifically, since the introduction of the Chinchilla scaling law, which models the final pre-training loss \( L(N, D) \) as a function of the number of model parameters \( N \) and the number of training tokens \( D \), models such as LLaMA2~\cite{touvron2023llama}, Mistral~\cite{jiang2023mistral}, and Gemma~\cite{team2024gemma} have applied this principle. However, merely increasing model size does not indefinitely enhance performance. A line of research focuses on the generalization behavior of over-parameterized neural networks. Recent experiments reveal that over-trained transformers exhibit an inverted U-shaped scaling behavior~\cite{belkin2019reconciling}, which existing empirical scaling laws cannot explain~\cite{niu2024beyond}. Following the success of scaling laws in large language models, many researchers have also sought to apply these principles to recommendation models, such as HSTU~\cite{zhai2024actions} and Wukong~\cite{gu2022wukong}.

%挑战1修改为仅performance,不再描述u-shape
%Future-work里面再提词表，这里说结构与数据质量
However, the current endeavor to transfer large model scaling laws to the analysis of recommendation systems faces two challenges: (1) Compared to the loss metrics typically examined in scaling laws, performance is the more pertinent outcome of interest for recommendation models. Although some studies have highlighted the relationship between performance and loss~\cite{du2024understanding}, these relationships remain unstable and lack quantitative fitting analysis. (2) In addition to data scale, data quality has increasingly become a focal point of concern, significantly impacting recommendation model datasets, which typically exhibit structural and collaborative characteristics and higher noise and redundancy levels. This positions data quality as a critical factor influencing outcomes, yet it remains unaddressed in the prevailing scaling laws for recommendation models.
%However, adapting scaling laws to recommendation systems faces two main challenges: (1) Unlike the loss functions typically examined in scaling laws, performance metrics like hit rate are more relevant for recommendation models. As shown in Figure~\ref{only_paint_results_NG}, the performance exhibits a sharper U-shaped curve as model size increases, indicating the need for a decay term in scaling laws to better capture this behavior. (2) Recommendation datasets often have variable vocabulary sizes and substantial structural differences across datasets, making it difficult to assess data quality based on scale alone. Current scaling laws in recommendation research primarily provide qualitative comparisons, lacking quantitative fitting analyses.

In response to Challenge (1), we propose the Performance Law for recommendation models. By fitting the hit rate (HR) and normalized discounted cumulative gain (NDCG) metrics of sequential recommendation models, we can utilize a few initial small-scale training points to infer the impact of model layers and item embedding dimensions on model performance. Toward Challenge (2), we introduced Approximate Entropy (ApEn) as a novel measure of data quality. We thoroughly validated, both theoretically and experimentally, the appropriateness of replacing data scale with data scale/approximate entropy, providing a more accurate substitute for the number of training tokens \( D \) in the original scaling law. Overall, all we need are the amount and quality of training data, along with initial experiments to balance parameters that have a core impact on model performance (such as the number of candidates, where performance can vary significantly between having 1,000 candidates versus 100). We obtained a surprisingly accurate correlation coefficient within large sequential models of different sizes and across various datasets. We also verified that the Performance Law can provide both the globally optimal and optimal performance for a given model size.

\section{Related Work}
\subsection{Sequential Recommendation}
The focus of recommendation systems has undergone significant transformations. Among these, sequential recommendation~\cite{he2023survey,guo2021dual,guo2023compressed} is a technique that aims to delve into and understand users' interest patterns by analyzing their historical interactions~\cite{tong2024mdap,wu2024survey}. Initially, techniques such as markov chain and matrix factorization were employed~\cite{10.1145/2911451.2911489}. 
However, with the emergence of neural networks, deep learning approaches have been developed
for sequential recommendation tasks.
Convolutional Neural Networks(CNN)~\cite{tang2018personalized}, Graph Neural Networks (GNN)~\cite{wu2019session} from GCE-GNN~\cite{10.1145/3397271.3401142} to SR-GNN~\cite{DBLP:journals/corr/abs-1811-00855} and URLLM~\cite{shen2024exploring}.
Multilayer Perceptron (MLP)~\cite{zhou2022filter} modeled SR from various perspectives, while the Diffusion-based method~\cite{xie2024breaking,xie2024bridging} and Fourier transform-based~\cite{wang2024denoising,han2023guesr} method are used to attempt noise reduction.
However, the compatibility of RNN~\cite{hidasi2015session} with sequences has ultimately garnered more attention. Early work like GRU4Rec~\cite{10.1145/3269206.3271761} and Caser~\cite{10.1145/3159652.3159656} were introduced to improve recommendation accuracy. Another notable technique in sequential recommendation is the attention mechanism. SASRec~\cite{8594844}, for instance, utilizes self-attention to independently learn the impact of each interaction on the target behavior. On the other hand, BERT4Rec~\cite{10.1145/3357384.3357895} incorporates bi-directional transformer layers after conducting pre-training tasks. 
Since LLaMA4Rec~\cite{luo2024integrating} and HSTU~\cite{zhai2024actions} demonstrated improvements in recommendation performance brought by large models and large datasets, It is meaningful to study how the model performance would change as the model size scales up.

%In recent years, graph neural networks (GNNs) have gained attention for their ability to capture higher-order relationships among items. GCE-GNN~\cite{10.1145/3397271.3401142} constructs local session graphs and leverages information from other sessions to create a dense global graph for modeling the current session. SR-GNN~\cite{DBLP:journals/corr/abs-1811-00855} employs gated GNNs in session graphs to capture complex item transitions. To address the issue of data sparsity, contrastive mechanisms have been adopted in some works. CL4SRec~\cite{9835621} and CoSeRec~\cite{DBLP:journals/corr/abs-2108-06479} propose data augmentation approaches to construct contrastive tasks, which help alleviate the sparsity problem.

\subsection{Scaling Law on Large Sequential Models}\label{rel2}
Scaling laws was first explored in the context of Large Language Models~\cite{kaplan2020scaling,yin2024entropy,he2023survey}. Specifically, since the introduction of the Chinchilla scaling law, which models the final pre-training loss \( L(N, D) \) as a function of the number of model parameters \( N \) and the number of training tokens \( D \), models such as LLaMA2~\cite{touvron2023llama}, Mistral~\cite{jiang2023mistral}, and Gemma~\cite{team2024gemma} have applied this principle.
Empirical evidence indicates that model performance consistently improves with increased model size and training data volume~\cite{kaplan2020scaling,khandelwal2019generalization,rae2021scaling,chowdhery2023palm}. Extensive experiments have explored neural scaling laws under various conditions, including constraints on computational budget~\cite{hoffmann2022empirical}, data limitations~\cite{muennighoff2023scaling} and regeneration~\cite{yin2024dataset}, and instances of over-training~\cite{gadre2024language}. These analyses employ a decomposition of expected risk, resulting in the fit of $L(N,D)=\left[\frac{N_c}{N}^{\frac{\alpha_N}{\alpha_D}}+\frac{D_C}{D}\right]^{\alpha_D},$
where $\alpha_N$ and $\alpha_D$ are parameters.
However, increasing the model size does not necessarily lead to better performance. Some studies have observed a decline in performance due to overfitting~\cite{belkin2019reconciling,nakkiran2021deep}.
Following theoretical analysis,~\cite{power2022grokking} and~\cite{niu2024beyond} empirically validated this point, underscoring the necessity for an enhanced understanding of scaling laws.
%Following theoretical analysis, \cite{power2022grokking} and \cite{niu2024beyond} empirically validated this point, further highlighting the need for expanding the current understanding of scaling laws.
\section{Preliminary and Theorem}\label{sec: preliminary}

\subsection{Problem Definition}
Sequential Recommendation (SR)~\cite{zhang2024unified,wangmf} focuses on predicting the next item a user is likely to interact with, given a sequence of previous interactions. This task is essential in personalized systems, as it considers the temporal dynamics and evolving preferences of users. Formally, let \( U \) denote the set of users and \( I \) the set of items. For a specific user \( u \in U \), we represent their interaction history as an ordered sequence \( S_u = [i_1, i_2, \ldots, i_n] \), where each \( i_k \in I \) indicates an item interacted with at time step \( k \). The objective is to model the probability distribution over \( I \), given \( S_u \), to predict the subsequent item \( i_{n+1} \). This problem is challenging due to the complex, dynamic patterns in user behavior, which are influenced by various contextual factors. Effective solutions must balance capturing short-term user intent and leveraging long-term preference structures.

\subsection{Preliminary}

\subsubsection{Scaling Law on Large Sequential Models}
As we introduced in the Related Work Section~\ref{rel2}, empirical evidence indicates that model performance consistently improves with increased model size and training data volume. These analyses employ a decomposition of expected risk, resulting in the following fit:
\begin{equation}
    L(N,D)=\left[\frac{N_c}{N}^{\frac{\alpha_N}{\alpha_D}}+\frac{D_C}{D'}\right]^{\alpha_D},
\end{equation}
where $\alpha_N$ and $\alpha_D$ are parameters. In some works, the scaling law formula might be simplified to $L'(N,D)=E+\frac{A}{N^\alpha}+\frac{B}{D^\beta}$ with parameter$\alpha$ and $\beta$. For example, in Chinchilla~\cite{hoffmann2022training}, the fitted parameters are:
\begin{equation}
E = 1.61,\ A = 406.4,\ B = 410.7\ \alpha = 0.34,\ \beta = 0.28
\end{equation}
However, directly applying the scaling laws of LLMs to the recommendation domain faces inconsistencies between language sequence modeling and recommendation sequence modeling. To address this, we continue to refine and supplement the scaling laws for recommendations to better adapt to the distribution characteristics of SR.

\subsubsection{ApEn Extension on Recommendation}
As mentioned above, although recommendation and text datasets contain sequential information, the structural differences and significant variations in user preferences prevent a straightforward application of token-based scaling laws for SR. Therefore, we introduce the Approximate Entropy~\cite{pincus1991approximate} (ApEn) factor to further enhance scaling laws in SR.

Specifically, ApEn is a statistical measure used to quantify the regularity and unpredictability of time series data, which is computed as follows:
First, for a given time series $\{x_i\}$ of length $N$ and parameters $m$ (embedding dimension) and $r$ (tolerance), we construct a $m$-dimensional vectors $\mathbf{X}_i = [x_i, x_{i+1}, \ldots, x_{i+m-1}]$ for $i = 1, \ldots, N-m+1$.
Subsequently, the distance between two such vectors $\mathbf{X}_i$ and $\mathbf{X}_j$ is defined as
\begin{equation} d[\mathbf{X}_i, \mathbf{X}_j] = \max_{0 \leq k < m} |x_{i+k} - x_{j+k}|. \end{equation}
Next, for a given tolerance $r$, the similarity measure $C_i^m(r)$ is calculated as
\begin{equation} C_i^m(r) = \frac{\left|\{ j \mid d[\mathbf{X}_i, \mathbf{X}_j] \leq r \}\right|}{N-m+1}. \end{equation}
The average similarity $\Phi^m(r)$ is then computed as
\begin{equation} \Phi^m(r) = \frac{1}{N-m+1} \sum_{i=1}^{N-m+1} \ln C_i^m(r). \end{equation}
The Approximate Entropy is finally defined as
\begin{equation} \mathrm{ApEn}(m, r, N) = \Phi^m(r) - \Phi^{m+1}(r). \end{equation}
In our subsequent calculation of ApEn, we set the tolerance \( r = 0 \). This decision is due to the unique nature of recommended items, where products with similar IDs can potentially convey entirely different meanings. From Equations above, it can be observed that Approximate Entropy \( ApEn \) has a positive correlation with the model's duplication rate, whereas conventional entropy tends to have a negative correlation with data duplication rate. Therefore, although \( ApEn \) is referred to as entropy, its trend is opposite to traditional entropy. To avoid confusion, we use \(  ApEn'= 1/ApEn\) as the final measure of Approximate Entropy.

\subsection{Theorem of Performance Law}
Before accurately predicting model performance, we need to consider factors such as model decay as the model size increases indefinitely, the relationship between model performance and loss, and how data scale correlates with token count and dataset Approximate Entropy (ApEn). For the first aspect, we introduce a \( \log( . ) \) decay term when the model layer \( H\) and embedding dimension \( d_{emb} \) is large, as proven in Theorem~\ref{theorem5}. Regarding the second aspect, we assume a linear relationship between model performance and loss, expressed as \( \text{Performance} = 1 - kL\), Performance = HR@10, NDCG@10,  For the final point, we provide a theoretical proof in Theorem~\ref{theorem3}, establishing the data scale $D$ is bounded by \( D = \#Tokens\cdot ApEn' \). Building on the aforementioned theorems, we ultimately fit the model's performance into the following equation:
\begin{equation}\label{Eq432890492}
\begin{aligned}
Perform&ance=w_1(\log N+\frac{p_1}{ N^{w_3}})+\\
&w_2(\log d_{emb}+\frac{p_1}{ d_{emb}^{w_4}})+\log D'+\frac{p_2}{ D'^{w_5}}
\end{aligned}
\end{equation}
where \( D' = \#Tokens\cdot ApEn' \) represents the data parameter, \( N \) is the number of model layers, and \( d_{emb} \) is the item embedding dimension. When controlling model size, we only adjust the number of layers and embedding dimension, which is why our formula includes only these two parameters. However, our theoretical framework can still be applied for analysis when other model parameters are modified. Below is the detailed proof of our performance law theory:

\begin{lemma}\label{lemma1}
In the first-order stationary Markov chain (discrete state space $X$) case, with r$< min( |$x- y| , x$\neq$y, x and y state space values), a.s. for any m
\begin{equation}
ApEn(\mathrm{m, r})=-\sum_{\mathrm{x\in X}}\sum_{\mathrm{y\in X}}\pi(\mathrm{x})\mathrm{p_{xy}} log(\mathrm{p_{xy}}).
\end{equation}
, where $\pi(\mathrm{x})$ is the stationary distribution of $x$.
%,其中\pi(\mathrm{x})为x的平稳分布。
\end{lemma}

\begin{proof}
By the definition of ApEn, $ApEn=-\mathrm{E}(\log(C_{1}^{m+1}(r)/C_{1}^{m}(r))).$
\begin{equation}
\begin{aligned}
ApEn&=-\mathrm{E}(\log(C_{1}^{m+1}(r)/C_{1}^{m}(r)))\\
&=-\mathrm{E}_j(\log\mathrm{P}(|x_{j+m}-x_{m+1}|\leq r\parallel\\
&\ \ \ \ \ \ \ \ \ \  |x_{j+k-1}-x_{k}|\leq r\ \mathrm{~for}\ k=1 , 2, . . , m)\\
&=-\mathrm{E}_j\ (\log\mathrm{P}(x_{j+m}=x_{m+1}\parallel x_{j+m-1}=x_{m}))\\
&=-\mathrm{E}_j(\log\mathrm{P}(x_{j+m}=x_{m+1}\parallel x_{j+k-1}=x_{k}\\
&\ \ \ \ \ \ \ \ \ \ \ \ \ \ \ \ \ \ \ \ \ \ \ \ \ \ \ \ \ \ \ \ \ \ \ \ \ \ \ \ \ \ \ \ \ \ \ \ \ \ \ \ \ \ \mathrm{~for}\ k=1 , 2, . . , m)\\
&=-\sum_{x\in X}\sum_{y\in X}\mathrm{P}(x_{j+m}=y , x_{j+m-1}=x)\ \\
&\ \ \ \ (\log \mathrm{P}(x_{j+m}=y , x_{j+m-1}=x)/\mathrm{P}(x_{j+m-1}=x) )\\
&\ \ \ \ =-\sum_{\mathrm{x\in X}}\sum_{\mathrm{y\in X}}\pi(\mathrm{x})\mathrm{p_{xy}} log(\mathrm{p_{xy}}).
\end{aligned}
\end{equation}
\end{proof}

\begin{lemma}\label{lemma2}
If $x_i,\ y_i,\ i=0,1,...,n$, and $\Sigma_{i=1}^{n}x_i=\Sigma_{i=1}^{n}y_i=1$, then
\begin{equation}
\sum_{i=1}^qx_i\log_r\frac1{x_i}\leq\sum_{i=1}^qx_i\log_r\frac1{y_i}    
\end{equation}
%若$$
\end{lemma}
\begin{proof}
\begin{equation}
\begin{aligned}
\sum_{i=1}^qx_i\log_r\frac1{x_i}\leq\sum_{i=1}^qx_i\log_r\frac1{y_i}& =\sum_{i=1}^qx_i\log_r(\frac{y_i}{x_i}) \\
&=\frac1{\ln r}\sum_{i=1}^qx_i\ln(\frac{y_i}{x_i}) \\
&\leq\frac1{\ln r}\sum_{i=1}^qx_i(\frac{y_i}{x_i}-1) \\
&=\frac1{\ln r}(\sum_{i=1}^qy_i-\sum_{i=1}^qx_i)=0
\end{aligned}
\end{equation} 
\end{proof}
With the aforementioned lemma established, we proceed to conduct a detailed analysis of data volume within the framework of the original Scaling Law. Importantly, the minimum encoding length L(C) provides a nuanced perspective by integrating data compression rates~\cite{sayood2017introduction,yin2024entropy}, thereby offering insights into both the scale and quality of the data. This refined approach permits a more thorough assessment of how data affects model performance. Guided by these insights, we propose the following theorem:

\begin{theorem}\label{theorem3}
Assuming that the user sequence can be modeled as a first-order aperiodic stationary Markov chain. If the user sequence is $S=\{S_u, u\in U\}$,  then the summary of minimum encoding lengths is given by:
\begin{equation}
|U|L(C)\geq (\Sigma_{i=1}^n|s_i|)\cdot ApEn'
\end{equation}
\end{theorem}

\begin{proof}
From Lemma~\ref{lemma2}, it follows that:
\begin{equation}\label{eq2489302}
\begin{aligned}
&L(C)=\sum_{\forall s_i}p(s_i)l_i=\sum_{\forall s_i}p(s_i)log_2\frac{1}{2^{-l_i}}\\
&\geq  H_2(\mathcal{S})= \sum_{\forall s_i} p(s_i)\log_2\frac1{p(s_i)},
\end{aligned}
\end{equation}
where \( H_2(\mathcal{S}) \) represents the entropy of the sequence in base 2, and Kraft's inequality~\cite{kraft1949device} \(\sum_{i=1}^{n} 2^{-l_i} \leq 1\) was utilized. Meanwhile, from Lemma~\ref{lemma1}, it follows that:
\begin{equation}\label{eq1234414}
\frac{(\Sigma_{i=1}^n|s_i|)}{ApEn(S)}=-\frac{(\Sigma_{i=1}^n|s_i|)}{\sum_{\mathrm{x\in I}}\sum_{\mathrm{y\in I}}\pi(\mathrm{x})\mathrm{p_{xy}} log(\mathrm{p_{xy}})},
\end{equation}
where $I$ is the set of all items, $p_{xy}$ is the transition probability in the sequence, and $ \pi(\mathrm{x})$ is the stationary distribution of the Markov chain, satisfying $\sum \pi(x) = 1$.
Combining Equation~\ref{eq2489302} and Equation~\ref{eq1234414}, the original proposition is equivalent to proving:
\begin{equation}
|U| \sum_{\forall s_i} p(s_i)\log_2\frac1{p(s_i)} \geq\frac{\Sigma_{i=1}^n|s_i|}{\sum_{\mathrm{x\in I}}\sum_{\mathrm{y\in I}}\pi(\mathrm{x})\mathrm{p_{xy}} log(\frac{1}{\mathrm{p_{xy}}})},
\end{equation}
, where it is necessary to ensure that $p(s_i)$ and $p_{xy}$ are minimal, which is more easily satisfied in recommendations with a large recommendation item vocabulary.  Combining the above equations, we need to demonstrate that:
\begin{equation}
\begin{aligned}
    \sum_{\forall s_i} p(s_i)\log_2\frac1{p(s_i)}& \sum_{\mathrm{x\in I}}\sum_{\mathrm{y\in I}}\pi(\mathrm{x})\mathrm{p_{xy}} log(\frac{1}{\mathrm{p_{xy}}})\\
    &\geq \frac{S_{max}\Sigma_{i=1}^n|s_i|}{|U| S_{max}}.
\end{aligned}
\end{equation}
We decompose this inequality into the following two inequalities for proof:
\begin{equation}\label{eq24029443}
        \sum_{\mathrm{x\in I}}\sum_{\mathrm{y\in I}}\pi(\mathrm{x})\mathrm{p_{xy}} log(\frac{1}{\mathrm{p_{xy}}}) \geq \frac{\Sigma_{i=1}^n|s_i|}{|U| S_{max}},
\end{equation}
\begin{equation}\label{eq48390585}
        \sum_{\forall s_i} p(s_i)\log_2\frac1{p(s_i)}\geq S_{max}.
\end{equation}
Let's first examine Equation~\ref{eq24029443}. Note that
\begin{equation}
    \sum_{\mathrm{x\in I}}\sum_{\mathrm{y\in I}}\pi(\mathrm{x})\mathrm{p_{xy}} log(\frac{1}{\mathrm{p_{xy}}})\geq \sum_{\mathrm{x\in I}}\sum_{\mathrm{y\in I}}\pi(\mathrm{x})\mathrm{p_{xy}}
\end{equation}
The right-hand side is non-zero only when it appears in the dataset, and each term has a minimum value of $\frac{1}{|U| \cdot S_{max}}$. Subsequently, we have:
\begin{equation}
    \sum_{\mathrm{x\in I}}\sum_{\mathrm{y\in I}}\pi(\mathrm{x})\mathrm{p_{xy}} log(\frac{1}{\mathrm{p_{xy}}})\geq \sum_{\mathrm{x\in I}}\sum_{\mathrm{y\in I}}\pi(\mathrm{x})\mathrm{p_{xy}}\geq\frac{\Sigma_{i=1}^n|s_i|}{|U| S_{max}}.
\end{equation}
On the other hand, using the inequality $\ln x \geq \frac{2(x-1)}{x+1}$ for $x \geq 1$, we have:
\begin{equation}
\begin{aligned}
        \sum_{\forall s_i} p(s_i)\log_2{\frac{1}{p(s_i)}}&\geq\ ln2\  \sum_{\forall s_i}p(s_i)\frac{2(1-p(s_i))}{p(s_i)+1}\\
        &\geq\ 2ln2\  \sum_{\forall s_i}(1-p(s_i))
\end{aligned}
\end{equation}
As long as the number of training cases $|U| > S_{max},$ and $2ln2>1$, $\sum_{\forall s_i} p(s_i) \log_2{\frac{1}{p(s_i)}} \geq S_{max}$. Combining these results completes the proof.
\end{proof}

\begin{lemma}~\cite{niu2024beyond}
Suppose $x=(x_{1},\ldots,x_{n})\in R^{S_{max} d_{emb}}$, then we have 
%$\min_{1\leq i\leq n}x_{i}-\log n-\frac{1}{n}\leq-LogSumExp(-x)<\min_{1\leq i\leq n}x_{i}.$
\begin{equation}
    1<Loss=\log Z_t+\frac1{Z_t}+\log l+\frac{1}{l}-c,
\end{equation}
\begin{equation}
    \frac{d\cdot V_{S_{max} d_{emb}}(\sqrt{\frac{n}{2\pi e}})}{\exp(\sqrt{\frac{S_{max} d_{emb}}{2\pi e}})}\leq Z_t\leq d\cdot V_n(\sqrt{\frac{S_{max} d_{emb}}{2\pi e}}),
\end{equation}
where $V_{S_{max} d_{emb}}(r)=\frac{\pi^{\frac{S_{max} d_{emb}}{2}}r^{S_{max} d_{emb}}}{\Gamma(1+\frac{S_{max} d_{emb}}{2})}$. The term \( S_{max} \times d_{emb} \) is used multiple times because it represents the maximum dimension of the input.
\end{lemma}

\begin{theorem}
 \label{theorem5}
There exist $w_3, w_3', w_4, w_4'$ such that
\begin{equation}
    \begin{aligned}
        log(d_{emb}^{w3}d^{w4})+&\frac{1}{d_{emb}^{w3}d^{w4}}\leq log(Z_t)+\frac{1}{Z_t}\\&\leq log(d_{emb}^{w3'}d^{w4'})+\frac{1}{d_{emb}^{w3'}d^{w4'}}
        \end{aligned}
\end{equation}\end{theorem}

\begin{proof}
First, assume $n=S_{max} d_{emb}$ we'll prove the left half of the inequality. This is because:
\begin{equation}
    log(d^0d_{emb}^0)+\frac{1}{d_{emb}^{0}d^{0}}=1\leq log(Z_t)+\frac{1}{Z_t}
\end{equation}
The right half can be proven using the following inequalities:
\begin{equation}
    \begin{aligned}
        Z_t\leq d\cdot V_n(\sqrt{\frac{n}{2\pi e}})&=\frac{d\pi^{\frac{n}{2}}\left(\sqrt{\frac{n}{2\pi e}}\right)^{n}}{\left(\Gamma(\frac{n}{2}+1\right))}\sim k\frac{d}{\sqrt{\frac{T_{max}d_{emb}}{2}\pi e}}\\
        log(Z_t)+\frac{1}{Z_t}&\leq O(log(d\cdot d_{emb}^{-\frac{1}{2}})+\frac{1}{d\cdot d_{emb}^{-\frac{1}{2}}})
    \end{aligned}
    \end{equation}
Here, we apply Stirling's approximation $\Gamma(z+1) \sim \sqrt{2\pi z} \left(\frac{z}{e}\right)^z$ as the proof of the rightmost inequality. This requires \(Z_t \leq 1\), which is common in the overfitting phenomenon when the model studied in this paper increases rapidly, as evidenced by the u-shape images in the experiments. This approach enables us to decompose the loss into the form \(\frac{1}{n} + \log(n)\) with a decay term, optimizing the fit for performance. Analogously to~\cite{wu2024performance}, we factorize the product into the structure of Equation~\ref{Eq432890492}, which represents the ultimate form of our fitting model.
\end{proof}
\section{Methodology}
Following the prior empirical study in SR~\cite{power2022grokking,niu2024beyond}, for all
experiments, we adopt the decoder-only transformer models as the backbone. Specifically, for each user \( u \), items and rating scores in the user behavior sequence \( X_u = \{i_1, r_1, \ldots, i_k, r_k, \ldots, i_n, r_n\} \) are firstly encoded into embeddings, forming \( \mathbf{e}_u = \{\mathbf{e}^{i_1}, \mathbf{e}^{r_1}, \ldots, \mathbf{e}^{i_k}, \mathbf{e}^{r_k}, \ldots, \mathbf{e}^{i_n}, \mathbf{e}^{r_n}\} \). After the embedding layer, we stack multiple Transformer
decoder blocks. At each layer l, query $Q$, key $K$, and value $V$ are projected from the same input
hidden representation matrix $H^l$.
We modified the $SiLU$ activation module and the $Rab$ positional encoding module within the standard transformer block to ensure the model's effectiveness. In Section~\ref{app2}, we will discuss the impact and improvements brought by these modules and demonstrate that our model is also applicable to standard transformers. 
Specifically, the modification is mainly on two core sub-layers: the spatial aggregation layer, and the pointwise transformation layer:

Firstly, the Spatial Aggregation Layer is defined as follows:
    \begin{equation} \label{eqn:attn-score}
    Attn(e^{\{i,r\}_k})V=SiLU(QK^T+Rab)V, k=1,...,n
\end{equation}
Upon deriving matrices query $Q$, key $K$, and value $V$, the spatial aggregation layer utilizes an attention mechanism to adjust V. This layer is distinguished by two key features: First, the SiLU activation replaces the standard softmax function typical in transformer models, effectively handling large and non-stationary data sets by removing the necessity for full normalization, thereby improving computational efficiency and stability. Second, a relative attention bias is applied, enriching the model with positional and temporal information, thus enhancing its ability to discern contextual interrelations and capture dependencies within the data.

Subsequently, the Pointwise Transformation Layer is defined as follows:
 \begin{equation} 
 \begin{aligned}
    e'^{\{i,r\}_k}=Norm(Attn(e^{\{i,r\}_k})&V \odot Gate(e^{\{i,r\}_k}),\\
    &k=1,...,n
    \end{aligned}
\end{equation}
Following the spatial aggregation layer, the pointwise transformation layer applies a transformation to each individual data point independently. Here, the gating weights $Gate(X)$ are combined with the normalized values $Norm(Attn(X)V(X))$ via a Hadamard product, effectively gating the transformed representations and allowing the model to selectively emphasize relevant entries while weakening less significant ones. The results are then transformed by a single-layer MLP.

Finally, we then calculate the similarity of this item \( \mathbf{e'}^{i_{k+1}} \) with those in the entire item pool \( \mathcal{I} \), retrieving the most similar item and storing it in the candidate set \( \mathcal{I}_c^u \). 
The loss function for the retrieval task is defined as:
\[
\small
\begin{aligned}
&\mathcal{L}_{\text{retrieval}} = \sum_{u \in \mathcal{U}} \sum_{k=2}^{n} 
\text{SampledSoftmaxLoss}(e'^{i_k})        \\
&= -\sum_{u \in \mathcal{U}} \sum_{k=2}^{n}\log\left(\frac{e^{\text{sim}(e'^{i_k}, e^{i_k})}}
{e^{\text{sim}(e'^{i_k}, e^{i_k})} + \sum_{j \in \mathcal{S}} e^{\text{sim}(e'^{i_k}, e^{j_k})}}\right),
\end{aligned}
\]
where $\mathcal{S}$ is a set of randomly sampled negative samples, and similarity measure $\text{sim}(\mathbf{a}, \mathbf{b}) = \mathbf{a} \cdot \mathbf{b}$, where $\cdot$ denotes the inner product operation.

%We consider all items that users interact with as positive items during the retrieval stage, even if users do not click on them at next stage. Therefore, we calculate the corresponding sampled softmax loss~\cite{sampledsoftmaxloss} at each position.
\section{Experimental Evaluation}\label{sec: experiment_setting}

\begin{figure*}[t]
  %%%%\vspace{-2mm}
  \centering
  \includegraphics[width=0.95\linewidth]{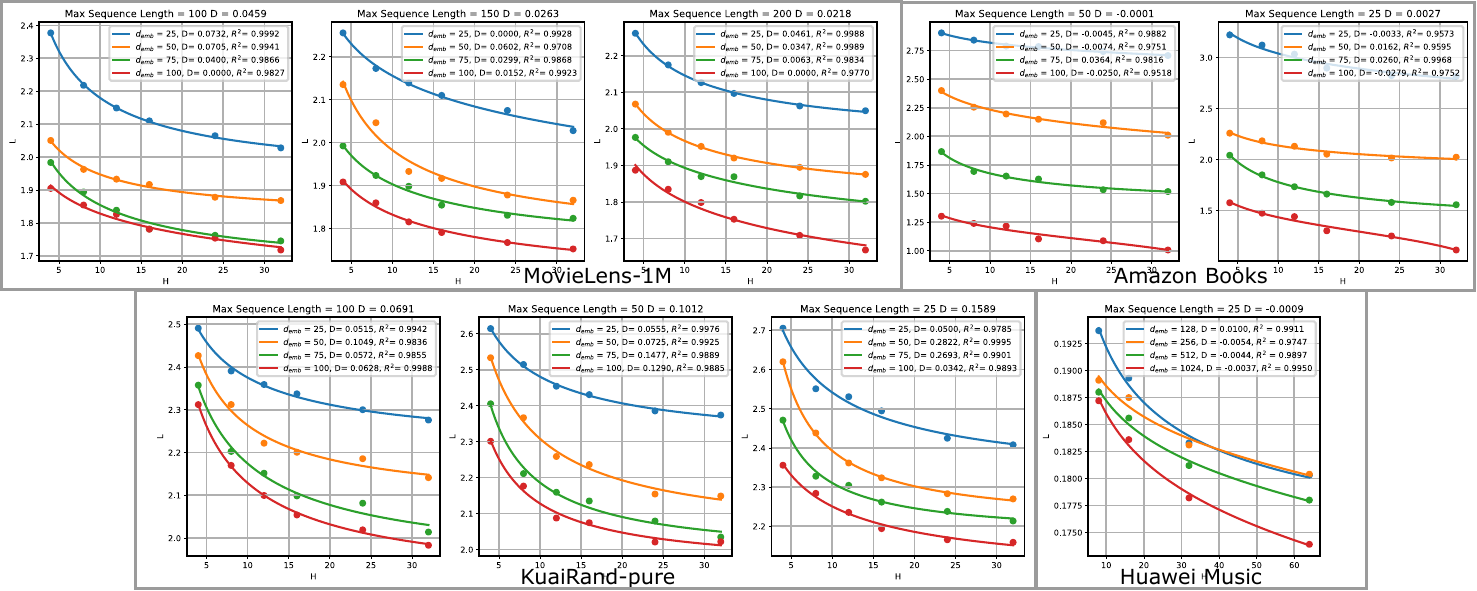}
  %%%%\vspace{-2mm}
  \caption{
  The relationship between model loss and the number of layers (horizontal axis, H), as well as the embedding dimensions (different colored lines, $d_{emb}$), the plot includes annotations of the coefficient of determination \(R^2\).}
  \label{3_datasets_in_1_loss}
\end{figure*}

\begin{table}[t]
    \centering
    \begin{tabular}{l|p{0.5cm}p{0.7cm}p{1.1cm}l|l}
    \hline
        Dataset &  $S_{max}$ & $\overline{S}$ & Tokens(T)  & ApEn/T & $1/\overline{D_L}$ \\ \hline
        ~ & 25 & 20.72 & 447407 & 1.646E-07
 & 0.1589 \\ 
        KR-pure & 50 & 33.70 & 570537 & 1.079E-07
 & 0.1012 \\ 
        ~ & 100 & 46.03 & 661028 & 8.874E-08
 & 0.0691 \\ \hline
        ~ & 100 & 74.07 & 505108 & 3.636E-08
 & 0.0750 \\
        ML-1m & 150 & 94.46 & 802493 & 2.111E-08
 & 0.0263 \\ 
        ~ & 200 & 109.4 & 1058511 & 1.551E-08
 & 0.0218 \\ \hline
        \multirow{2}{*}{AMZ} & 50 & 11.57 & 8044865  & 1.243E-09 & 0.0021 \\ 
        & 25 & 10.18 & 7076238 & 1.554E-09 & 0.0030 \\ \hline
        Huawei & 25 & 17.01 & 3.275E8  & 3.572E-10 & 0.0001 \\ \hline
    \end{tabular}
  \caption{
  The basic information for different datasets, where \(S\) denotes the sequence length, along with the fitted data parameter \(1/\overline{D_L}\). It's relationship versus \(Apen/T\) is illustrated in Figure~\ref{loss_HR_NG}.}
\end{table}

\subsection{Datasets}

\begin{figure*}[t]
  %%%%\vspace{-2mm}
  \centering
  \includegraphics[width=0.95\linewidth]{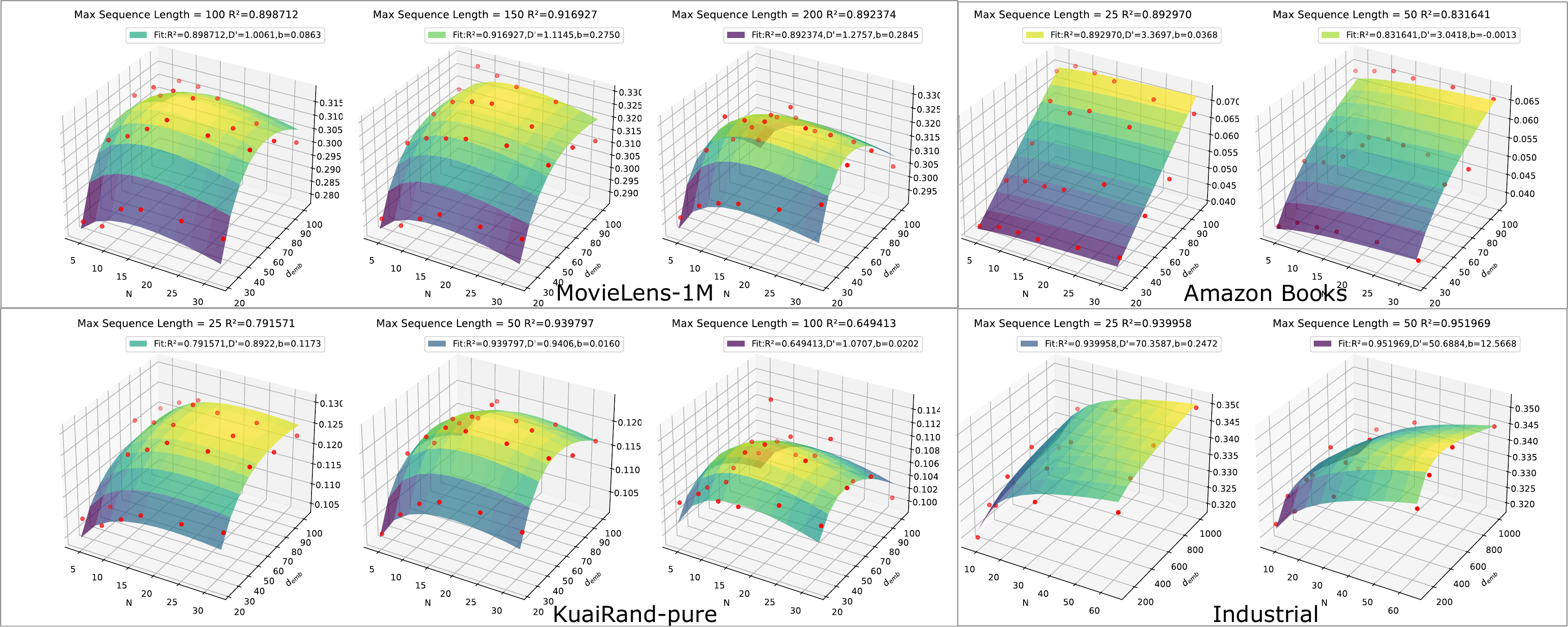}
  %若用pdf，则编译会十分缓慢，最后再改为pdf
  %%%%\vspace{-2mm}
  \caption{
  The relationship between model HR performance and the number of layers (x-axis, N), as well as the embedding dimensions (y-axis, $d_{emb}$), the plot includes annotations of the fitted parameter $w$.}
  \label{only_paint_results_HR}
\end{figure*}
\begin{figure}[t]
  %%%%\vspace{-2mm}
  \centering
  \includegraphics[width=1.1\linewidth]{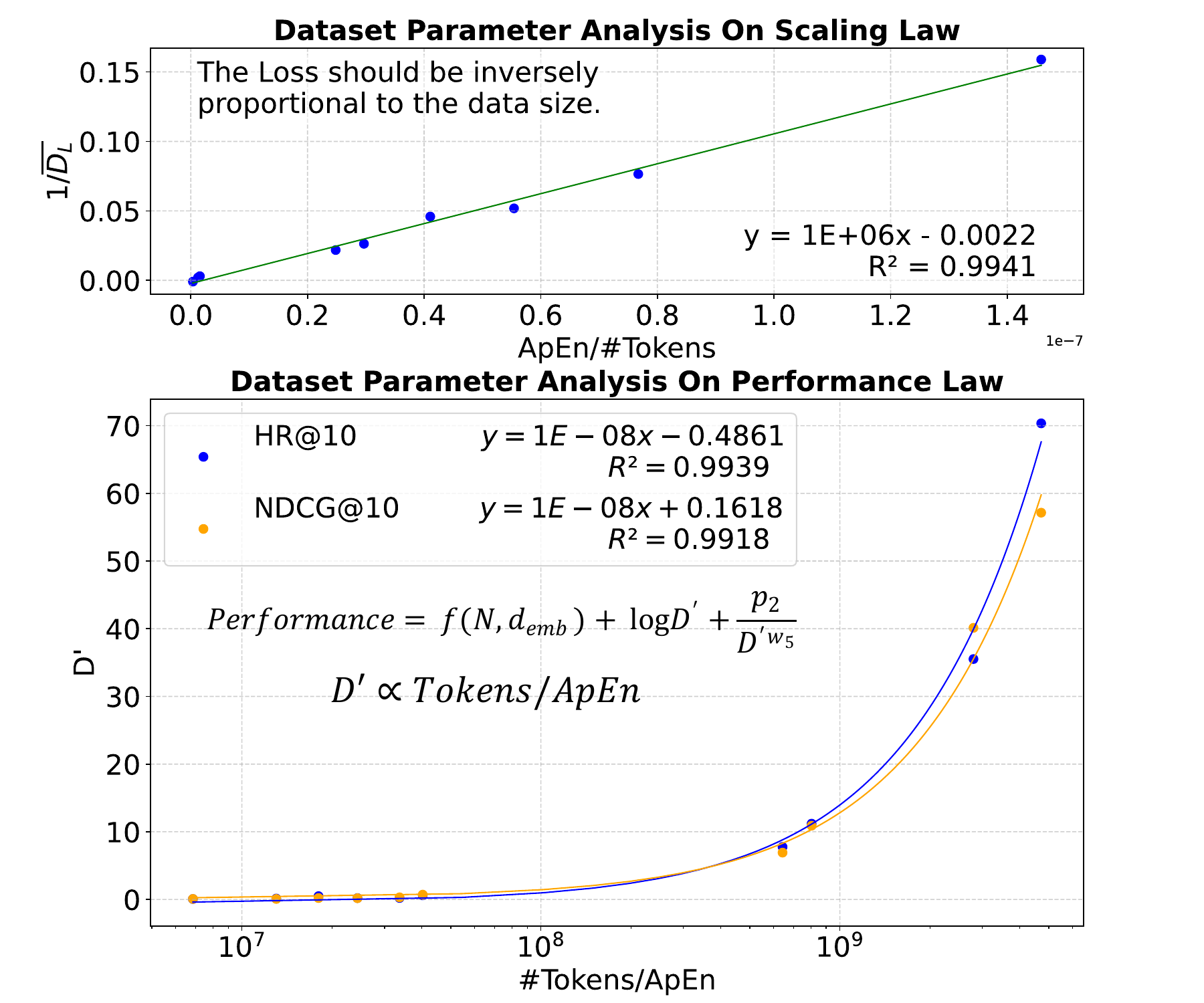}
  %%%%\vspace{-2mm}
  \caption{
  The linear correlation between parameter $D$ and Tokens/Apen. The upper figure validates this relationship within the context of the Scaling Law Loss, while the lower figure verifies it within the Performance Law Metric.}
  \label{loss_HR_NG}
\end{figure}
To demonstrate the performance of our proposed model, we conducted experiments on three publicly available datasets and a private dataset, whose introduction is as follows:
\begin{itemize}[left=0pt]
\item \textbf{MovieLens-1M~\cite{harper2015movielens}\footnote{\url{grouplens.org/datasets/movielens/1m/}}} : This dataset is a standard benchmark in recommendation systems research, containing 1 million ratings from 6,000 users on 4,000 movies, featuring ratings from 1 to 5, along with demographic and movie metadata.

\item \textbf{Amazon Books~\cite{mcauley2013hidden}\footnote{\url{snap.stanford.edu/data/web-Amazon-links.html}}}: Amazon Books, a subset of the Amazon review dataset, includes user reviews and ratings for books on Amazon. It provides explicit (ratings) and implicit (reviews) feedback, making it valuable for studying recommendation models.

\item \textbf{KuaiRand-pure~\cite{gao2022kuairand}\footnote{\url{http://kuairand.com}}}: This dataset offers extensive interaction data from a popular short video platform, capturing behaviors like watching, liking, and commenting. It includes user and content features such as demographics, preferences, metadata, and engagement metrics, essential for developing and evaluating recommendation algorithms.

\item \textbf{Huawei Music}: The Huawei Music dataset is a large proprietary collection featuring extensive user interaction data, challenging existing scaling laws. It includes detailed user behaviors and music metadata, providing a significant resource for testing the applicability of scaling laws to larger datasets in recommendation systems.
\vspace{-4mm}
\end{itemize}

%Metric?

\subsection{Recommendation Fit on Scaling Law}

\begin{figure*}[t]
  %%%%\vspace{-2mm}
  \centering
  \includegraphics[width=0.95\linewidth]{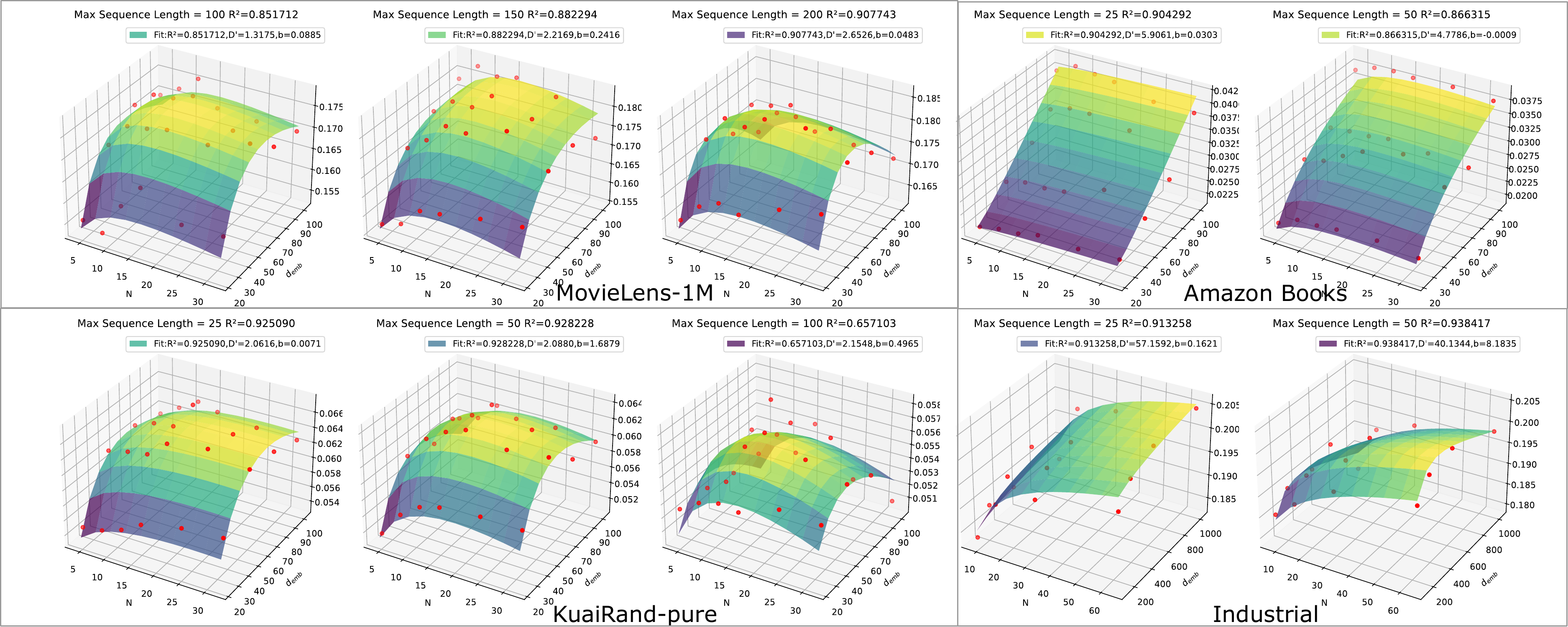}
  %若用pdf，则编译会十分缓慢，最后再改为pdf
  %%%%\vspace{-2mm}
  \caption{
  The relationship between model NDCG performance and the number of layers (x-axis, N), as well as the embedding dimensions (y-axis, $d_{emb}$), the plot includes annotations of the fitted parameter $w$.}
  \label{only_paint_results_NG}
\end{figure*}

To validate the soundness of our Theory 1, we aim to demonstrate: (1) our model's adherence to Scaling Laws, and (2) the appropriateness of using both Approximate Entropy (ApEn) and token count to assess data scale. Specifically, we first examine the alignment of the model's loss curve with Scaling Laws. We then analyze Loss across various datasets, fitting it to traditional Scaling Laws to obtain data parameters. If the fitted data exhibits a clear linear relationship with the combination of ApEn and token count, it will sufficiently substantiate the validity of Theory~\ref{theorem3}.

On one side, the results of Scaling Law fitting are illustrated in Figure~\ref{3_datasets_in_1_loss}. Due to space constraints, we present only the fitting for KuaiRand-1K, constraining its maximum sequence length to 100, 50, and 25. All $R^2$ values exceed 0.95, demonstrating that our model adheres to the scaling law trend, enabling further analysis. 
On the other side, the fitting results of the data parameters are illustrated in Figure~\ref{loss_HR_NG}. In the figure, we display the fitted relationship between the average value of the dataset parameter $D'$ and the dataset size (evaluated by $\#$Tokens)/ApEn. A clear linear growth trend emerges, aligning perfectly with our predictions. This analysis is consistent across all four datasets, further validating the effectiveness of our approach.

\subsection{Recommendation Fit on Performance Law}

\begin{table*}[t]
    \centering
    \caption{The basic information for different datasets, where \(S\) denotes the sequence length, along with the fitted data parameter \(D\) in different metrics (HR@10, NDCG@10. It's relationship versus \(Apen/T\) is illustrated in Figure~\ref{only_paint_results_HR} and Figure~\ref{only_paint_results_NG}.}
    \begin{tabular}{l|lllll|ll|ll}
    \hline
        Metric & ~ & ~ & ~ & ~ & ~ & HR@10 & ~ & NDCG@10 & ~ \\ \hline
        dataset & $S_{max}$ & $\overline{S}$ & tokens & ApEn & tokens/ApEn & D' & w5 & D' & w5 \\ \hline
        \multirow{3}{*}{ Kuairand-Pure } & 25 & 20.72 & 447407 & 0.074 & 6074772.573 & 0.0595 & 0.4864 & 0.0737 & 0.5260 \\
         & 50 & 33.70 & 570537 & 0.062 &9266477.180 & 0.1316 & 0.9353 & 0.0717 & 0.5117 \\ 
         & 100 & 46.03 & 661028 & 0.059 & 11268803.27 & 0.4874 & 7.1546 & 0.2056 & 2.5941 \\ \hline
        \multirow{3}{*}{ ML-1m } & 100 & 74.07 & 505108 & 0.018 & 27496352.75 & 0.2033 & 2.6036 & 0.1811 & 1.9953 \\ 
       & 150 & 94.46 & 802493 & 0.017 & 47372668.24  & 0.2167 & 2.7461 & 0.3099 & 4.0034 \\ 
         & 200 & 109.4 & 1058511 & 0.016 & 64464738.12 & 0.6067 & 4.3223 & 0.7077 & 18.460 \\ \hline
        \multirow{2}{*}{Amazon-books} & 50 & 11.57 & 8044865 & 0.010 & 804486500.0 & 11.197 & 0.9999 & 10.897 & -6.4985 \\ 
        & 25 & 10.18 & 7076238 & 0.011 & 643294363.6 & 7.7641 & 0.8972 & 6.9176 & -6.0527 \\ \hline
        \multirow{2}{*}{Huawei-Music} & 50 & 26.69 & 513878761 & 0.110 & 4714484046 & 70.350 & -4.2350 & 57.153 & -3.6724 \\ 
         & 25 & 17.01 & 327509107 & 0.117 & 2798983907 & 35.522 & -12.740 & 40.134 & -29.948 \\ \hline
    \end{tabular}
\end{table*}

To validate the soundness of our Theory 2, we need to verify that: (1) Our formula demonstrates a strong fitting correlation coefficient under different model parameters. (2) Using ApEn and token count as measures of data quality, there is still a strong linear correlation with the data parameter $D'$.

On one side, our empirical results confirm the applicability of performance laws by demonstrating consistent quality improvements across diverse datasets. Leveraging Approximate Entropy (ApEn) as a key metric, our model exhibits robust correlations between data quality measures and output performance. This result suggests that ApEn, combined with our existing scaling metrics, serves as a reliable predictor of model efficacy, effectively bridging theoretical insights with practical outcomes.
Conversely, our analyses reveal that this relationship holds true regardless of dataset configuration. Whether applied to varied sequence lengths or alternative data structures, the model maintains its integrity and performance, illustrating its adaptability. The alignment of scaling and performance laws offers a comprehensive validation framework, underscoring our model's capacity to achieve high-quality outputs consistently.

This dual validation—achieved through the complementary application of scaling and performance analyses—not only fortifies our theoretical predictions but also guides future enhancements. It marks a significant step forward in understanding and optimizing model behavior.

\subsection{Applications of Performance Law}
\subsubsection{Application 1: Global and Local Optimal Parameter Search}
\begin{table*}[t]
    \centering
    \caption{The performance of the model under different parameter settings, with the bottom row indicating the optimal parameters as calculated using the Performance Law.}
    \begin{tabular}{cccccc|cccccc}
    \hline
    \multicolumn{6}{c|}{Global optimal solution}&\multicolumn{6}{c}{Local optimal solution (with N=8 and d=64 as the baseline).}\\
        N & d & NG@10 & NG@50 & HR@10 & MRR & N & d & NG@10 & NG@50 & HR@10 & MRR \\ \hline
        64 & 128 & 0.1974 & 0.2578 & 0.3415 & 0.1688   & 256 & 2 & 0.1632 & 0.2248 & 0.2939 & 0.1392\\ 
        64 & 256 & 0.2019 & 0.2623 & 0.3481 & 0.1726 & 128 & 4 & 0.1758 & 0.2371 & 0.3111 & 0.1504 \\ 
        64 & 370 & 0.2035 & 0.2639 & 0.3504 & 0.1737 &64 & 8 & 0.1773 & 0.2381 & 0.3118 & 0.1520  \\ 
        64 & 512 & 0.2032 & 0.2636 & 0.3502 & 0.1737 & 51 & 10 & 0.1758 & 0.2365 & 0.3092 & 0.1507 \\ 
        64 & 1024 & 0.1981 & 0.259 & 0.3448 & 0.1688 & 32 & 16 & 0.1704 & 0.2305 & 0.3007 & 0.1462 \\ \hline
        64 & 603 & \textbf{0.2040} & \textbf{0.2644} & \textbf{0.3512} & \textbf{0.1744} & 48 & 14 & \textbf{0.1777} & \textbf{0.2383} & \textbf{0.3121} & \textbf{0.1523} \\ \hline
    \end{tabular}
\end{table*}
An intriguing and practical application of the performance law lies in predicting the performance gain from model expansion techniques. Due to the inclusion of a decay term in our performance law, it becomes feasible to potentially achieve a global optimum. From the aforementioned fitting, we derived two distinct performance law formulations from HR and NG, as follows:
\begin{equation}
\begin{aligned} 
HR=&0.50(34N^{-0.0001}+ln(N))-19.01(3*d_{emb}^{-0.0365}\\
&+ln(d_{emb}))+2.5D'^{-12.74}+ln(D')+628.12
\end{aligned}
\end{equation}
\begin{equation}
\begin{aligned} 
NG&=0.63(31N^{-0.0001}+ln(N))-24.85(31d_{emb}^{-0.0395}\\
&+ln(d_{emb}))+3D'^{-0.0410}+ln(D')+743.9
\end{aligned}
\end{equation}

Given the immense size of the dataset, the effectiveness continuously improves with an increase in the model's depth. Thus, we can only obtain the optimal solution for \(d_{\text{emb}}\) on a global scale. Nonetheless, we can still determine the optimal configuration of the model depth \(N\) and \(d_{\text{emb}}\) under model size constraints (to align with practical online models, we select \(d_{\text{emb}}=64\) and \(N=8\) as the baseline). From the results in the Table, it is evident that our formulation can achieve performance maxima both globally and under constraints.

\subsubsection{Application 2: Exploring Scaling Law Potential Among Framework}\label{app2}
Another application of the Performance Law is to observe the potential performance gains when scaling up the model. We conducted experiments and fitting analyses on three different frameworks (HSTU, LLaMA2, and SASRec), evaluated at different precisions (float32 and bfloat16). We conducted experiments on the smallest dataset (ML-1m) to facilitate the models in reaching their optimal upper bounds more easily. Their optimal results were fitted with respect to the coefficient and exponent terms of $N$ and $d$, respectively denoted as $(w_1, w_3)$ and $(w_2, w_4)$. In the expressions \( w_1 \frac{1}{N^{w_3}} \) and \( w_2 \frac{1}{d^{w_4}} \), when the coefficient terms \( w_1 \) and \( w_2 \) are negative, smaller values of \( w_3 \) and \( w_4 \) indicate poorer scaling-up potential of the model, and vice versa. Our results, along with their performance under optimal parameters, are presented in Table~\ref{parameter}. In all the fittings conducted here, both \( w_1 \) and \( w_2 \) are positive values. As observed from the table, the model's performance closely aligns with the magnitude trend of \( w_3 \) and \( w_4 \), which further underscores the accuracy of our fitting in predicting the model's performance.

%模型精度敏感性
\begin{table*}[t]
\small
    \centering
    \caption{Comparison of Model Parameters and Performance Across Different Precisions}
    \label{parameter}
    \begin{tabular}{cc|ccc|ccc}
    \hline
    
    ~   &Precision  & ~ & Float32 & ~ & ~ & Bfloat16 & ~ \\ 
   ~    & Model & HSTU & LLaMA & Sasrec & HSTU & LLaMA & Sasrec \\  \hline
 \multirow{2}{*}{Parameter}      & $w_3$ & -1.0403 & -1.4638 & 0.0737 & -0.3178 & -0.4844 & 1.013 \\ 
       & $w_4$ & 0.1425 & 0.0359 & 0.4578 & 0.2341 & 0.2186 & 0.6273 \\ \hline
        \multicolumn{2}{c|}{Model Performance (HR@10)} & 0.3322 & 0.3459 & 0.3021 & 0.3319 & 0.3367 & 0.2938 \\\hline
    \end{tabular}
\end{table*}
\section{Conclusion}\label{sec: discussion}
In conclusion, this paper addressed two critical challenges in the domain of SR: the discrepancy between model loss and actual performance, and the adverse effects of data redundancy on model efficacy. Our introduction of the Performance Law for SR models marks a significant advancement in the theoretical exploration of performance metrics in relation to SR models, emphasizing the role of data quality over sheer quantity. By fitting performance metrics such as hit rate HR and NDCG to transformer-based SR models and proposing the use of ApEn, we offer a novel approach that enhances the predictive accuracy of model outcomes. This framework not only provides a more nuanced understanding of how data quality impacts model performance but also facilitates the strategic balancing of computational resources to achieve optimal results. Our findings demonstrate that meaningful predictions regarding model performance can be achieved across varying scales of datasets and model sizes, thereby delivering valuable insights into optimizing recommendation systems in practical applications.


\bibliography{KDD25}
\bibliographystyle{icml2024}

%%%%%%%%%%%%%%%%%%%%%%%%%%%%%%%%%%%%%%%%%%%%%%%%%%%%%%%%%%%%%%%%%%%%%%%%%%%%%%%
%%%%%%%%%%%%%%%%%%%%%%%%%%%%%%%%%%%%%%%%%%%%%%%%%%%%%%%%%%%%%%%%%%%%%%%%%%%%%%%
% APPENDIX
%%%%%%%%%%%%%%%%%%%%%%%%%%%%%%%%%%%%%%%%%%%%%%%%%%%%%%%%%%%%%%%%%%%%%%%%%%%%%%%
%%%%%%%%%%%%%%%%%%%%%%%%%%%%%%%%%%%%%%%%%%%%%%%%%%%%%%%%%%%%%%%%%%%%%%%%%%%%%%%

%\newpage
%\appendix
%\onecolumn

%\section{You \emph{can} have an appendix here.}

%You can have as much text here as you want. The main body must be at most $8$ pages long.
%For the final version, one more page can be added. If you want, you can use an appendix like this one.  

%The $\mathtt{\backslash onecolumn}$ command above can be kept in place if you prefer a one-column appendix, or can be removed if you prefer a two-column appendix.  Apart from this possible change, the style (font size, spacing, margins, page numbering, etc.) should be kept the same as the main body.
%%%%%%%%%%%%%%%%%%%%%%%%%%%%%%%%%%%%%%%%%%%%%%%%%%%%%%%%%%%%%%%%%%%%%%%%%%%%%%%
%%%%%%%%%%%%%%%%%%%%%%%%%%%%%%%%%%%%%%%%%%%%%%%%%%%%%%%%%%%%%%%%%%%%%%%%%%%%%%%

\end{document}